\documentclass[oneside]{article}
\pdfoutput=1 
\usepackage{arxiv}
\usepackage[utf8]{inputenc} 
\usepackage[T1]{fontenc}    
\usepackage{hyperref}       
\usepackage{url}            
\usepackage{booktabs}       

\usepackage{amsmath}
\usepackage{amsthm}
\usepackage{amsfonts}
\usepackage{amssymb}
\usepackage{graphicx}
\usepackage{enumerate}
\usepackage{mathrsfs}
\usepackage{braket}
\usepackage{hyperref}
\usepackage{nccmath}
\usepackage{nicefrac}       
\usepackage{microtype}      

\newtheorem{theorem}{Theorem}

\newtheorem{proposition}[theorem]{Proposition}

\newcommand{\mbf}[1]{\mathbf{#1}}
\newcommand{\sbf}[1]{\boldsymbol{#1}}

\title{A Solution for Large Scale Nonlinear Regression with High Rank
and Degree at Constant Memory Complexity via Latent Tensor Reconstruction}

\author{
  Sandor Szedmak \\ 
  Department of Computer Science\\
  Aalto University\\
  Espoo, Finland \\
 \texttt{sandor.szedmak@aalto.fi} \\
  \And
  Anna Cichonska \\ 
  Department of Computer Science \\
  Aalto University \\
  Espoo, Finland \\
  \texttt{anna.cichonska@aalto.fi} \\
  \And
  Heli Julkunen \\ 
  Department of Computer Science\\
  Aalto University\\
  Espoo, Finland \\
 \texttt{heli.julkunen@aalto.fi} \\
  \And
  Tapio Pahikkala \\ 
  Department of Computer Science\\
  University of Turku\\
  Turku, Finland \\
 \texttt{aatapa@utu.fi} \\
  \And
  Juho Rousu \\ 
  Department of Computer Science\\
  Aalto University\\
  Espoo, Finland \\
 \texttt{juho.rousu@aalto.fi} \\
}

\begin{document}

\maketitle

\begin{abstract}
This paper proposes a novel method for learning highly nonlinear,
multivariate functions from  examples. Our method takes advantage of
the property that continuous functions can be approximated by
polynomials, which in turn are representable by tensors.  
Hence the function learning problem is transformed into a tensor
reconstruction problem, an inverse problem of the tensor 
decomposition. Our method incrementally builds up the unknown tensor 
from rank-one terms, which lets us control the complexity of the
learned model and reduce the chance of overfitting. For learning the
models, we present an efficient gradient-based algorithm that can be
implemented in linear time in the sample size, order, rank of
the tensor and the dimension of the input. In addition to regression,
we present extensions to classification, multi-view learning and
vector-valued output as well as a multi-layered formulation. 
The method can work in an online fashion via processing mini-batches of
the data with constant memory complexity. Consequently,
it can fit into systems equipped only with limited resources
such as embedded systems or mobile phones.        
Our experiments demonstrate a favorable accuracy and running time
compared to competing methods. 
\end{abstract}

\section{Introduction}

The rise of deep learning methods have showed that complex,
nonlinear problems can be learned in practice if the model families
have sufficient capacity to cover the underlying
interdependence. Extending the domain of known approaches further 
could increase our capability to tackle new challenging
problems. In this paper, we study the problem of learning noisy, highly
nonlinear multivariate functions from large data sets ($> 100000$
examples). We assume that the target functions to be learned are real
valued and their domain is a bounded subset of finite dimensional
Hilbert space. The corresponding training data contains a set of
input-output pairs, points of the Hilbert space and the function
values (scalar or vector) in those points. 
To construct the estimates of the functions, we further assume that they are
continuous or can be approximated by a continuous function with a
fixed $L_{\infty}$ norm based tolerance. The continuity assumption
allows us to exploit the  Stone-Weierstrass theorem and its
generalizations \cite{MR0385023}, namely those function can be
approximated  by polynomials on a compact subset with an  
accuracy not worse than a given arbitrary small error. Since every
polynomial can be represented by tensors, the function learning problem
can be formulated as a tensor reconstruction problem. This reconstruction
is a certain inverse of the tensor decomposition, i.e. HOSVD, in which
a tensor is built up, for example, from one-rank tensors, see methods
such as CANDECOMP/PARAFAC tensor decomposition by
\cite{Kolda09tensordecompositions},  \cite{Lathauwer2000}, or
\cite{Golub2013}.     

The tensor reconstruction approach allows to reduce the number of
parameters required to represent the polynomials. An arbitrary multivariate
polynomial defined on the field of real numbers can be described by
$\binom{n+d}{n}$ parameters, where $n$ is the number of variables, $d$
is the maximum degree. Thus the complexity relating to the size of the
underlying tensor is $O(n^d)$, which increases exponentially in the
number of parameters. A properly chosen variant of the tensor
decomposition can significantly reduce that complexity. To this end, we
consider polynomials of this type   
\begin{equation} 
f(\mbf{x}) = \sum_{t=1}^{n_t} \prod_{d=1}^{n_d} \ell_d^{(t)}(\mbf{x}),\
\mbf{x}\in \mathcal{X}\subseteq \mathbb{R}^{n},
\end{equation}
where $\ell_d^{(t)}(\mbf{X})$ for all $d$ and $t$ are linear forms.
This representation is also called polynomial factorization
or decomposition problem \cite{Kaltofen:1990:CPG:77763.77768}.
The main difference between the HOSVD and the linear form
factorization is that in the former the orthogonality of the
vectors defining the linear forms is generally assumed, but in the
factorization, 
only the linear independence is expected. A significant advantage of
the linear form based representation is that the polynomial function
depends only linearly on its parameters. The basic model of this paper is
outlined in Table \ref{table:Ttof_StoT}.

The motivation of the learning approach presented
in this paper comes from several  sources. One of the influencing
models is the {\it Factorization machine (FM)} which applies a special
class, the symmetric polynomials, see \cite{zbMATH00967945}, to
approximate the underlying, unknown nonlinear functions. It was introduced by
\cite{Rendle:2010:FM:1933307.1934620} and later on, further extensions
have been published, such as the higher order case (HOFM)   
by \cite{Blondel:2016:HFM:3157382.3157473}. The FM itself also has
close connection to the Polynomial Networks,
\cite{Livni:2014:CET:2968826.2968922}, as highlighted by 
\cite{Blondel:2016:PNF:3045390.3045481}. Our approach can be viewed as a
reformulation of factorization machines, removing the limiting
assumption of the symmetry imposed on the polynomials, and hence
extending the range of learnable functions. 
Furthermore, the latent tensor models
\cite{ESAllman2009,JMLR:v15:anandkumar14b,JMLR:v16:huang15a} also
provide important background to the 
proposed method. In those approaches, the moments, up to third order, of
the input data are exploited to yield the tensor structure. In our
method, the decomposition of the latent tensor is conditioned on the
output data as well, thus it is also a cross-covariance based method.


\begin{table}[t]
\begin{equation*}
\renewcommand{\arraystretch}{1.2}
\begin{array}{c|c}
\text{Polynomial} & \text{Tensor} \\
\text{from tensor} & \text{from sample examples} \\ \hline \hline
\begin{array}{c}
\multicolumn{1}{l}{\text{Given}:\ \mbf{T}\ \text{tensor},\ \mbf{x}} \\
 \multicolumn{1}{l}{\text{Output}:\ y} \\ \hline
\\
\quad \mbf{T} \ \qquad \mbf{x} \\
\quad \downarrow \ \ \qquad \downarrow\\
f(\mbf{x}) = \braket{\mbf{T}, \otimes^{n_d} \mbf{x}} \Rightarrow y \\
\end{array}
&
\begin{array}{c}
\multicolumn{1}{l}{{\text{Given}:}\ \{(y_i,\mbf{x}_i)|i=1,\dots,m\}} \\ 
\multicolumn{1}{l}{{\text{Output}:}\ \mbf{T}\ \text{tensor}} \\
\hline
\\ 
\{y_i\} \qquad \qquad \{\mbf{x}_i \} \\
\downarrow \quad \qquad \qquad \downarrow \\
\min_{\mbf{T}} \sum_i ||y_i-\braket{\mbf{T},\otimes^{n_d}
\mbf{x}_i}||^2 \Rightarrow \mbf{T} 
\end{array}
\end{array}
\renewcommand{\arraystretch}{1.0}
\end{equation*}
\caption{The general scheme of Latent Tensor Reconstruction based regression}
\label{table:Ttof_StoT}
\end{table}

The latent tensor reconstruction (LTR) based function learning can be
interpreted as a certain algebraic alternative of a multi-layered Deep
Artificial Neural Network (DNN). In DNN, the linear subproblems are
connected by  nonlinear activation functions, e.g. sigmoid, into a
complete learning system \cite{Goodfellow:2016:DL:3086952}. By
contrast, in LTR, the nonlinear subproblems are integrated by simple
linear transfer functions. The latter model satisfies the basic
principle of the dynamic programming  \cite{Bellman1957}, that is, if
there are $N$ subproblems processed sequentially, then after computing
$n<N$ subproblems, the optimization of subproblem $n+1$ does not
influence the optimum solution computed on the previous $n$ ones. This
fact eliminate the need of backpropagation generally required in the
training of an DNN. The LTR model can also be applied as a special
type of recurrent module embedded into a  DNN based learner. In this
paper we focus only on standalone applications.  


Our contributions are summarized in the following points: 

- LTR incrementally estimates the unknown function, and
monotonically reduces the approximation error. This allows
to control the exploration of the hidden structure of the function
and the underlying tensor as well.

- The proposed method provides a nonlinear, polynomial regression
with computational complexity linear in the order of tensor $n_d$, in
its rank, $n_t$, in the sample size, $m$, and in the number of variables $n$. 
The LTR is well suited to learn functions whose variables are interrelated.

- By applying mini-batch based processing, LTR also has constant
memory complexity similarly to DNN learners. The parameters are stored
$O(n_t n_d n)$ memory in the iterative procedure, and in each step
only a fixed size subset, a mini-batch, of the data is required. 

{\bf Notation and conventions:}
In the text $\otimes$ is used to denote the  tensor product of
vectors, $\braket{,}$ and $||\ ||$ stand for inner product and 
the induced norm in a Hilbert space $\mathcal{H}$. The notation $\braket{,}$ is
also applied for the Frobenius inner product of tensors, $\circ$
denotes the pointwise product of tensors with the same shape of any order. The
pointwise $k$ power of a vector or a matrix $\mbf{x}$ is given by
$\mbf{x}^{\circ k}$. The trace of square matrix $\mbf{A}$ is denoted by
$\text{tr}(\mbf{A})$, and $\text{vec}(\mbf{A})$ stands for the column
wise vectorization of matrix $\mbf{A}$. $\mbf{1}_m$ denotes a vector
of $m$ dimension with all components equal to $1$. The word {\it
polynomial} can also mean {\it polynomial function} depending on the
context. 


\section{Background}

\subsection{Data representation}

In the learning problem  we have a sample of examples given by
input-output pairs $\mathcal{S}=\{ (\mbf{x}_i,\mbf{y}_i) |
i=1,\dots,m,\ \mbf{x}_i \in \mathbb{R}^{n},\ y_i\in \mathbb{R}^{n_y}\}$ taken from
an unknown joint distribution of input and output sources. The rows of
the matrix $\mbf{X}\in \mathbb{R}^{m \times n}$ contain the vectors
$\mbf{x}_i$, and similarly the rows of $\mbf{Y}$ hold the output
vectors, $\mbf{y}_i$, for all $i$. In the first part of the paper we
deal with the case where $n_y=1$, and in Section
\ref{sec:vector_valued} the extension to the vector valued case is presented.          



\subsection{Polynomials and Tensors}

\label{sec:polynomials}

{\it Polynomials} are fundamental tools of the computational
algebra. Let $x_1,\dots x_{n_d}$ be a set of variables which can take
values of a field $\mathcal{K}$, in our case from $\mathbb{R}$. A monomial is a 
product $x_1^{\alpha_1}\cdot \ldots \cdot x_{n_d}^{\alpha_{n_d}}$ where the
powers are non-negative integer numbers. With the help of $n_d$-tuples
$\sbf{\alpha}=(\alpha_1,\dots, \alpha_{n_d})$ and
$\mbf{x}=(x_1,\dots,x_{n_d})$ a monomial can be written as
$\mbf{x}^{\sbf{\alpha}}$, thus the $\sbf{\alpha}$ can be applied as a
fingerprint of the monomial. Let $\mathcal{P}$ be a finite set of tuples  
with type of $\sbf{\alpha}$, then a polynomial $f$ is a finite 
linear combination of monomials defined on the same set of variables,
$f=\sum_{\sbf{\alpha}\in \mathcal{P}} C_{\sbf{\alpha}} \mbf{x}^{\sbf{\alpha}}$, where
$C_{\sbf{\alpha}}\in \mathbb{R}$. 
The degree of a monomial is the sum of the powers of its variables, $\sum_{d=1}^{n_d}\alpha_d$. A degree of a polynomial is the maximum degree of the monomials contained. A polynomial is homogeneous if all of its monomials have the same degree.  

A non-homogeneous polynomial with degree $n_d$ can be derived from a
homogeneous one of degree $n_d+1$ by substituting a real number, e.g. $1$ into
one of the variables, thus $(x_0,x_1,\dots,x_{n_d})$ is transformed into
$(1,x_1,\dots,x_{n_d})$. Based on this fact in the sequel we assume
that the polynomials are homogeneous.      

Let $\mathcal{X}_k=\mathbb{R}^{n_k},k=1,\dots,n_d$ be a set of finite
dimensional real vector spaces, and for each $k$ the set
$\mathcal{X}^{*}_k$ denotes the dual of $\mathcal{X}_k$, the space of
linear functionals acting on 
$\mathcal{X}_k$. A tensor $\mbf{T}$ as a multilinear form of order
$n_d$ can be defined as an element of $\otimes^{n_d}_k 
\mathcal{X}^*_k$. If a basis is given in each of $\mathcal{X}^*_k$,
then $\mbf{T}$ can be represented by an $n_d$-way array
$[T_{j_1,\dots,j_{n_d}}]_{j_1=1,\dots,j_{n_d}=1}^{n_1,\dots,n_{n_d}}$. In
the sequel, we might also refer to the $n_d$-way arrays as tensor as well. 
The tensors form a vector space of dimension $\prod_{k=1}^{n_d} n_k$. 

We can write the multilinear function, $f: \otimes^{n_d}_{k=1}
\mathcal{X}_k \rightarrow \mathbb{R}$, as
\begin{equation}
\begin{array}{ll}
\displaystyle f(\mbf{x}_1,\dots,\mbf{x}_{n_d}|\mbf{T}) 
=\sum_{j_1=1,\dots,j_{n_d}=1}^{n_1,\dots,n_{n_d}}
T_{j_1,\dots,j_{n_d}} x_{1,j_1}\dots x_{n_d,j_{n_d}}.
\end{array}
\end{equation}

If it is not stated otherwise, in this paper we generally assume that
the vector spaces $\mathcal{X}_k,k=1,\dots,n_d$ are the same and
$\mathcal{X}=\mathbb{R}^{n}$, thus the indexes of vector spaces can be
dropped 
\begin{equation}
\label{eq:tensor_poly}
f(\mbf{x}|\mbf{T})=\sum_{j_1=1,\dots,j_{n_d}=1}^{n}T_{j_1,\dots,j_{n_d}}
x_{j_1}\dots x_{j_{n_d}}.
\end{equation}
Tensor $\mathscr{T}$ is symmetric if for
any permutation $\sigma$ of the indexes $j_1,\dots,j_{n_d}$ the
identity $T_{j_1,\dots,j_{n_d}}=T_{j_{\sigma(1)},\dots,j_{\sigma(n_d)}}$
holds. In some cases this symmetry is called super-symmetry, see
\cite{Kolda09tensordecompositions}.  

The space of symmetric tensors of order $n_d$ defined on
$\mathcal{X}$ as vector space is isomorphic to the space of homogeneous
polynomials of degree $n_d$ defined on $\mathcal{X}$. 
Therefore any homogeneous polynomial with variables $x_1,\dots,x_{n_d}$ and with 
degree $n_d$ as multilinear form defined over the field of real numbers can
be represented by the help of a symmetric tensor.   

\subsection{Representation of multilinear functions}

The tensor $\mbf{T} \in \otimes^{n_d} \mathbb{R}^{n}$ may be
given in a decomposed form \cite{Lathauwer2000,Kolda09tensordecompositions}   
\begin{equation}
\begin{array}{ll}
\mbf{T}& = \displaystyle \sum_{t=1}^{n_t} \lambda_{t} \mbf{p}^{(t)}_1 \otimes \dots \otimes
\mbf{p}^{(t)}_{n_d}, \\
\text{s.t.} & ||\mbf{p}^{(t)}_d||=1,\ \mbf{p}^{(t)}_d\in
\mathbb{R}^{n}, \\
& t=1,\dots,n_t,\ d=1,\dots,n_d.
\end{array}
\end{equation}
This representation is generally not unique, see for example 
\cite{Lathauwer2000,deSilva:2008:TRI:1461964.1461969}.  
By replacing $\mbf{T}$ with its decomposed form, the polynomial 
function of (\ref{eq:tensor_poly}) turns into the following expressions 
\begin{equation} 
\label{eq:function_dot_product}
\begin{array}{ll}
f(\mbf{x})& =  \displaystyle \sum_{t=1}^{n_t} \lambda_{t} \braket{\mbf{p}^{(t)}_1 \otimes \dots \otimes
  \mbf{p}^{(t)}_{n_d}, \mbf{x} \otimes \dots \otimes
  \mbf{x}  }.  \\
& = \displaystyle \sum_{t=1}^{n_t} \lambda_{t} \braket{\mbf{p}^{(t)}_1,\mbf{x}} \cdot \dots\:
  \cdot  \braket{\mbf{p}^{(t)}_{n_d},\mbf{x}},
\end{array}
\end{equation}
where we exploit the well known identity connecting the inner product
and the tensor products, \cite{Golub2013}. This form only 
consists of terms of scalar factors, where each scalar is the value of
a linear functional acting on the space $\mathcal{X}$. This transformation
eliminates the potential difficulties which could arise in working
directly with full tensors. Observe that the function $f$ is linear in
each of the vector valued parameters, $\mbf{p}_{d}^{(t)},t=1,\dots,n_t,\
d=1,\dots,n_d$. 

\subsection{Polynomial regression, the generic form}

We start on a generic form of multivariate polynomial regression
described by a tensor. Let the degree of the polynomial be
equal to $n_d$, then by the help of (\ref{eq:tensor_poly}) we can write  
\begin{equation} 
\begin{array}{ll}
\min & \displaystyle \sum_{i=1}^{m}\| y_i -
f(\mbf{x}|\mbf{T}) \|^2 \\
\text{w.r.t.} & \displaystyle \mbf{T}\in \otimes^{n_d} \mathbb{R}^{n},\ j_1,\dots,j_{n_d} \in \{1,\dots,n\}. 
\end{array} 
\end{equation}
Here the polynomial is assumed to be homogeneous, for the inhomogeneous case see Section
\ref{sec:polynomials}. Representing functions by polynomials is an
attractive approach supported by the Stone-Weierstrass type theorems,
\cite{MR0385023}. Those theorems connect continuous functions to 
polynomials. A general form of those theorems is given here for real,
separable Hilbert spaces \cite{PMPrenter1970},
\begin{theorem}Let $\mathcal{H}$ be a real, separable Hilbert space. The
family of continuous polynomials on $\mathcal{H}$, restricted to a
compact set $\mbf{K} \subset \mathcal{H}$, is dense in the set
$\mathcal{C}(\mathcal{K})$ of continuous functions mapping
$\mathcal{H}$ into $\mathcal{H}$, and restricted to $\mathcal{K}$,
where $\mathcal{C}(\mathcal{K})$ carries the uniform norm topology.  
\end{theorem}
Informally, this theorem states that to any continuous function $F$ there is
a polynomial $f$ to be arbitrary close $F$.


Since the tensor representing a multivariate polynomial defined on the
field of the real numbers is symmetric, the number of parameters in the
polynomial is equal to $\binom{n+N-1}{n}$ for the homogeneous
case. This number grows exponentially in the number of degree and the
variables.  To find a sounding polynomial approximation for large scale
problems we need to reduce the dimension of the parameter space in a
way which preserves 
the approximation flexibility but the computational complexity in the
number of variables, and in the degree, is linear. 

\subsection{Factorization machines}

{Factorization machines (FM)} \cite{Rendle:2010:FM:1933307.1934620} apply a special
class of polynomials, the symmetric polynomials, see \cite{zbMATH00967945},  to
estimate nonlinear functions. FMs were extended to higher order case (HOFM)  
by \cite{Rendle:2012:FML:2168752.2168771} and 
\cite{Blondel:2016:HFM:3157382.3157473}. FMs are closely related to 
the polynomial networks, see \cite{Livni:2014:CET:2968826.2968922} as highlighted by
\cite{Blondel:2016:PNF:3045390.3045481}. 

The model of factorization machine can be described via a compact, recursive form \cite{Blondel:2016:HFM:3157382.3157473} 
\begin{equation}
\label{eq:hofm_vector}
\begin{array}{l}
\displaystyle \mbf{y}(\mbf{X}) =\sum_{d=1}^{n_d} 
\left(\frac{1}{d}\sum_{r=1}^{d} (-1)^{r+1} \mbf{A}_{d-r} \circ
\mbf{D}^{(d)}_r \right)\mbf{1}_{n_t},
\end{array}   
\end{equation}
where $\mbf{A}_{d-r} \in \mathbb{R}^{m \times n_t}$ for $r=0,\dots,d$,
$(\mbf{A}_0)_{it}=1$ for all $i,t$, and $\mbf{D}^{(d)}_r=
\mbf{X}^{\circ^d}(\mbf{P}^{(t)\circ^d})^T$ for all $d,r$, and  $\mbf{P}^{(t)}\in
\mathbb{R}^{n_t\times n}$. This model can be derived from the
{\it fundamental theorem of symmetric polynomials} \cite{zbMATH00967945},
which states that every symmetric polynomial can be expressed as a polynomial of the
power sums, if the basic field of the polynomials contains the rational numbers. In the case of the FM the power sums are given as  $(\mbf{D}^{(d)})_{i,t}=\sum_{j=1}^{n} (x_{ij}p_{tj})^{d}$. 
A function (polynomial) is symmetric if its value is invariant on any
permutation of the variables, e.g. $f(x,y)=xy$. Symmetric polynomials
can well approximate symmetric functions. In the non-symmetric cases
the approximation capability of the factorization machine is limited,
see the examples in Table \ref{tbl:simplepoly_results}.

\subsection{Kernel Ridge regression}

The Kernel Ridge Regression (KRR) is a popular alternative to resolve the polynomial regression problem with small to medium-sized data. It can be stated in a matrix form
\begin{equation} 
\begin{array}{ll}
\min_{\alpha} & \|\mbf{y}- \mbf{K}\sbf{\alpha}\|^2 + C(\sbf{\alpha}^T\mbf{K}\sbf{\alpha})^2 \\
\end{array}
\end{equation}
For polynomial regression, in the KRR  problem, the kernel $\mbf{K}$ is chosen as the polynomial kernel
$ K_{ij}= (\braket{\mbf{x}_i,\mbf{x}_j}+b)^{n_d},$
where $b\in \mathbb{R}$ is the bias term. For further details of kernel-based learning, see for example \cite{ShaweTaylor2004}. The parameter space of KRR has dimension $m$ equal to the size of sample and independent from the number of variables. The latter property has an advantage when the the number of variables is large. However solving this type of problem assuming dense kernels requires $O(m^3)$ time, and $O(m^2)$ space complexity, which is a bottleneck for processing large data sets, although several approaches have been proposed to reduce those complexities, e.g. Nystr\"{o}m approximation and random Fourier features \cite{NIPS2007_3182}. 

\section{Latent tensor reconstruction}

We first construct the basic problem where the target function is assumed to be scalar valued.  
The idea of the learning task is outlined in Table
\ref{table:Ttof_StoT}. Let $f^{(t)}(\mbf{x}) = \prod^{n_d}_{d=1} \braket{\mbf{p}^{(t)}_d,\mbf{x}}$,
then the corresponding optimization problem
can be formulated similarly to the Kernel Ridge Regression by assuming
least square loss and Tikhonov type regularization of the parameters. 
\begin{equation}
\label{eq:basic_ltr}
\renewcommand{\arraystretch}{1.5}
\begin{array}{ll}
\min & \displaystyle \frac{1}{m}\sum_{i=1}^{m}||y_i-
\sum_{t=1}^{n_t}\lambda_{t} f^{(t)}(\mbf{x}) ||^2  
+ \frac{C_p}{n_t n_d n}
\sum_{t=1}^{n_t}\sum_{d=1}^{n_d}||\mbf{p}^{(t)}_d||^2\\  
\text{w.r.t.} &
\lambda_{t},\ \mbf{p}^{(t)}_1,\dots,\mbf{p}^{(t)}_{n_d},\ t=1,\dots,n_t,\\
\end{array}
\renewcommand{\arraystretch}{1}
\end{equation}
It is important to mention here that without the regularization of the parameters, the tensor one-rank approximation problem might be ill-posed for certain tensors \cite{deSilva:2008:TRI:1461964.1461969}. 
We can make an important statement about this problem. 
The objective function,  if all variables,
$\lambda_{t},(\mbf{p}^{(t)}_1,\dots,\mbf{p}^{(t)}_{n_d}),\
t=1,\dots,n_t$, except one are fixed then in that variable the 
function within the norm is linear, consequently convex. Therefore the
objective function in that variable is also convex since the norm and
the summation preserve the convexity. 

\subsection{LTR representation vs. factorization machine}

In our polynomial representation a matrix is assigned to each
variable (component) of the input vectors, namely
for a fixed index of the variables $j\in \{1,\dots,n\}$ we have $j \rightarrow
\mbf{P}_{j}=[\mbf{p}^{(t)}_1,\dots,\mbf{p}^{(t)_{n_d}}]\in \mathbb{R}^{n_d\times n_t}$. In the Factorization Machine (FM)
\cite{Blondel:2016:HFM:3157382.3157473}
a vector $p_j \in
\mathbb{R}^{n_t}$ is assigned to the variable indexed by $j$.


The assignment of the LTR allows to decouple the factors of the
polynomial, thus the formulation becomes more transparent. The linear
dependence on the parameters allows us to use a simple gradient descent
algorithm. The increased parameter capacity requires larger sample
size to properly estimate the parameters. However when the sample is
sufficiently large the LTR performs significantly better that the FM
and KRR, see the examples in Section \ref{sec:experiments}.   


\subsection{Learning algorithm}

\begin{table}[htbp]
\begin{tabular}{c|c}
\begin{minipage}{0.48\linewidth}
\begin{equation*}
\label{eq:basic_algorithm}
\renewcommand{\arraystretch}{1.2}
\begin{array}{l}
\text{\bf The basic rank wise algorithm} \\ \\

\text{\bf Given: } \text{a sample: } \{(\mbf{x}_i,y_i),\ i=1,\dots,m \}, \\
\text{Let}\ y^{(1)}_i=y_i,\ i=1,\dots,m. \\
\text{\bf For}\ t=1\ \text{\bf to}\ n_t\ \text{\bf do}  \\
\quad \text{Solve rank-one subproblem} \\ 
\quad
\fbox{ $
\begin{array}{ll}
\min & \displaystyle \sum_{i=1}^{m}||y_i^{(t)}-
\lambda_{t}f^{(t)}(\mbf{x}_i) ||^2 \\
& + \frac{C_p}{n_d n} \sum_{d=1}^{n_d}||\mbf{p}^{(t)}_d||^2\\ 
\text{w.r.t.} & \lambda_{t},\ \mbf{p}^{(t)}_1,\dots,\mbf{p}^{(t)}_{n_d}, \\
\end{array}
$}
\\
\quad \text{Optimum solution:} \\
\quad \quad  \lambda_{t}^{*},\ \mbf{p}^{(t)*}_1,\dots,\mbf{p}^{(t)*}_{n_d} \\ 
 \quad \text{Deflation of the output:} \\
 \displaystyle \quad \quad y^{(t+1)}_i=y^{(t)}_i- \lambda_{t}^{*}\prod^{n_d}_{d=1}
\braket{\mbf{p}^{(t)*}_d,\mbf{x}_i}, \forall i \\
\end{array}
\renewcommand{\arraystretch}{1}
\end{equation*}
\end{minipage}
&
\begin{minipage}{0.48\linewidth}
\begin{equation*}
\renewcommand{\arraystretch}{1.2}
\begin{array}{l}
\text{\bf Solution scheme for the Subproblem} \\ \\
\text{\bf Given: } \{\mathcal{I}_b,b=1,\dots,n_b \}\\
\quad  \text{index sets of mini-batches}, \\
n_e\ \text{ Number of epochs}, \\ 
\gamma\ \text{ Learning speed}, \\
\text{Initialize the variables: } \lambda_{t}, \left( \mbf{p}_d^{(t)}
  \right)_{d=1}^{n_d}, \\
\quad \text{the gradients:} \dfrac{\partial f}{\partial \lambda_{t}}, \left( \dfrac{\partial
  f}{\partial \mbf{p}^{(t)}_k} \right)_{d=1}^{n_d}, \text{and} \\
\quad \text{the aggregated gradients for ADAM: }
v_{\lambda},(\mbf{v}_{d})_{d=1}^{n_d},  \\ 
\text{\bf For}\ e=1\ \text{\bf to}\ n_e\ \text{\bf do} \\
\quad  \text{\bf For}\ b=1\ \text{\bf to}\ n_b\ \text{\bf do} \quad 
\text{Process mini-batches} \\ 
\qquad \text{Compute mini-batch relative gradients} \\
\qquad \text{Update aggregated gradients}, \\
\qquad \text{Update variables}\ \lambda_{t},
\mbf{p}_1^{(t)},\dots,\mbf{p}_1^{(N)}   
\end{array}
\renewcommand{\arraystretch}{1}
\end{equation*}
\end{minipage}
\end{tabular}
\caption{The basic rank-wise algorithm of the scalar valued LTR, ({\it
on the left}), and of a solution scheme of subproblem by momentum based gradient update, e.g. ADAM, ({\it on the right}).}
\label{tab:algorithms_frame}
\end{table}

The summary of the basic algorithm and the solution to the 
rank-one subproblem is outlined in Table
\ref{tab:algorithms_frame}. The basic rank-wise algorithm follows 
the scheme of a Singular Value Decomposition by adding the best
fitting rank-one layers to the tensor under construction. In the
rank-one subproblem, the 
mini-batch wise processing of the data can be interpreted as a
piece-wise polynomial approximation, and also a certain variant of the spline
methods  \cite{Friedman91multivariateadaptive}, where the splines are
connected via momentum-based \cite{BTPolyak1964,nesterov2005smooth},
gradient updates, e.g. ADAM \cite{Kingma2014}.  
 
The selection of the mini-batches could follow an online processing
scheme, or if the entire data set is available then they can be
uniformly randomly chosen out of the full set by imitating a
stochastic gradient scheme.


\subsection{Vector-valued output}
\label{sec:vector_valued}

The outputs of the learning problem might be given as vectors
$\{\mbf{y}_i | i=1,\dots,m,\ \mbf{y}_i \in \mathcal{Y} \subseteq
\mathbb{R}^{n_y}\}$. Let the matrix $\mbf{Y}\in \mathbb{R}^{m \times n_y}$
contain the vectors $\mbf{y}_i$ in its rows for all $i$. We can extend
the basic scalar valued regression (\ref{eq:basic_ltr}) into a vector
valued one  
where the boxes highlights the changes
applied on the scalar valued case.    
\begin{equation}
\label{eq:vectorapprox}
\begin{array}{@{}l@{\:}l@{}}
 \min & \displaystyle \frac{1}{m n_y}\sum_{i=1}^{m} \Big \| \fbox{$ \mbf{y}_i $}-
\sum_{t=1}^{n_t}\lambda_{t} f^{(t)}(\mbf{x}) \fbox{$ \mbf{q}^{(t)} $} \Big
\|^2 \\
& \displaystyle + \frac{C_p}{n_t n_d n}
\sum_{t=1}^{n_t}\sum_{d=1}^{n_d}||\mbf{p}^{(t)}_d||^2 + \fbox{$ \displaystyle \frac{C_q}{n_t n_y} \sum_{t=1}^{n_t}
||\mbf{q}^{(t)}||^2 $}  \\  
\text{w.r.t.} &
\lambda_{t}, \fbox{$ \mbf{q}^{t} $}\in \mathbb{R}^{n_y}, \mbf{p}^{t}_d\in \mathbb{R}^{n}, \\ & t=1,\dots,n_t,\ d=1,\dots,n_d. \\
\end{array}
\end{equation}

The terms in the loss function of the problem can be rewritten in a
tensor and in a pointwise product form as well 
\begin{equation*}
\renewcommand{\arraystretch}{1.4}
\begin{array}{l}
\left \| \mbf{y}_i-
\sum_{t=1}^{n_t} \lambda_{t}
\Braket{ \otimes_{r=d}^{n_d} \mbf{p}^{(t)}_d
,  \otimes^{n_d} \mbf{x}_i } \mbf{q}^{(t)} \right \|^2    
= \left \| \mbf{Y}- \sum_{t=1}^{n_t}\lambda_{t}
\left(\circ^{n_d}_{d=1} \mbf{X}\mbf{p}^{(t)}_d \right) \otimes
\mbf{q}^{(t)} \right \|^2. 
\end{array}
\renewcommand{\arraystretch}{1}
\end{equation*}
In the vector valued formulation the vectors $\mbf{q}^{(t)}$ play the
role of certain nonlinear principal components of the output
vectors. The coordinates of the output vectors with respect to those
principal components are expressed by the polynomials defined on the
input vectors.     

\subsection{Matrix representations}

Let the following matrices be formed from the data and the parameters:
\begin{equation}
\label{eq:ltr_variables}
\renewcommand{\arraystretch}{1.8}
\begin{array}{l|lll}
\text{Parameters} & \mbf{P}_d& \displaystyle =\left [ \mbf{p}_d^{T} \right ]_{t=1}^{n_t} & \in
\mathbb{R}^{n_t \times n}, \\ 
& \mbf{Q} & \displaystyle = \left [ \mbf{q}^{(t)T} \right ]_{t=1}^{n_t} & \in
\mathbb{R}^{n_t \times n_y}, \\ \hline
\text{Scaling factors} & \sbf{\lambda} & \displaystyle = [\lambda_t]_{t=1}^{n_t} & \in \mathbb{R}^{n_t},  \\
& \mbf{D}_{\lambda} & = \text{diag}(\sbf{\lambda}) & \in
\mathbb{R}^{n_t \times n_t}, \\ \hline
\text{Polynomial} \\
\quad \ \text{complete} & \mbf{F} & \displaystyle = \circ_{d=1}^{n_d}\mbf{X}
\mbf{P}_d^{T} & \in \mathbb{R}^{m \times n_t}, \\
\quad \ \text{without }d & \mbf{F}_{\setminus d} & \displaystyle = \circ_{k=1,k\ne d}^{n_d}\mbf{X} \mbf{P}_k^{T} & \in
\mathbb{R}^{m \times n_t}, \\ \hline
\text{Error} & \mbf{E} & =\mbf{Y} - \mbf{F}\mbf{D}_{\lambda}\mbf{Q} & \in
\mathbb{R}^{m \times n_y}.
\end{array}
\renewcommand{\arraystretch}{1}
\end{equation}
We can write Problem (\ref{eq:vectorapprox}) in a compact matrix form.
\begin{equation}
\label{eq:ltr_matrix}
\renewcommand{\arraystretch}{1.3}
\begin{array}{ll}
\min & \displaystyle \dfrac{1}{2 m n_y}\left \| \mbf{Y}- \mbf{F}\mbf{D}_{\lambda}\mbf{Q}
\right \|_F^2 
+ \dfrac{C_{p}}{2 n_tn_d n}\sum_{d=1}^{n_d}||\mbf{P}_d||_F^2 +
\dfrac{C_{q}}{2n_t n_y}||\mbf{Q}||_2^2\\   
\text{w.r.t.} &
\mbf{D}_{\lambda},\ \mbf{Q},\ \mbf{P}_1,\dots,\mbf{P}_{n_d}.
\end{array}
\renewcommand{\arraystretch}{1}
\end{equation}
Based on the notation introduced in (\ref{eq:ltr_variables}), the partial
derivatives to the problem (\ref{eq:ltr_matrix}) can be computed
in matrix form as well. Let $h$ denote the entire objective function
of (\ref{eq:ltr_matrix}), hence we can write    
\begin{equation}
\label{eq:matrix_gradient}
\renewcommand{\arraystretch}{1.4}
\begin{array}{ll}
\nabla_{\mbf{P}_d} h& \displaystyle = - (\mbf{F}_{\setminus d}^{T} \circ
\mbf{D}_{\lambda}\mbf{Q}\mbf{E}^{T})\mbf{X} +
\frac{C_{p}}{n_tn_d n}\mbf{P}_d, d=1,\dots,n_d. \\
\nabla_{\sbf{\lambda}} h & \displaystyle = - (\mbf{F}^{T} \circ
\mbf{Q}\mbf{E}^{T})\mbf{1}_m, \\
\nabla_{\mbf{Q}} h & \displaystyle = - \mbf{D}_{\lambda} \mbf{F}^{T} \mbf{E} +
\frac{C_{q}}{n_t n_y}\mbf{Q}, \\ 
\end{array}
\renewcommand{\arraystretch}{1}
\end{equation}

\subsection{Complexity of the algorithm}

\begin{figure}[tbp]
\begin{center}
  \includegraphics[width=6in,height=2in]{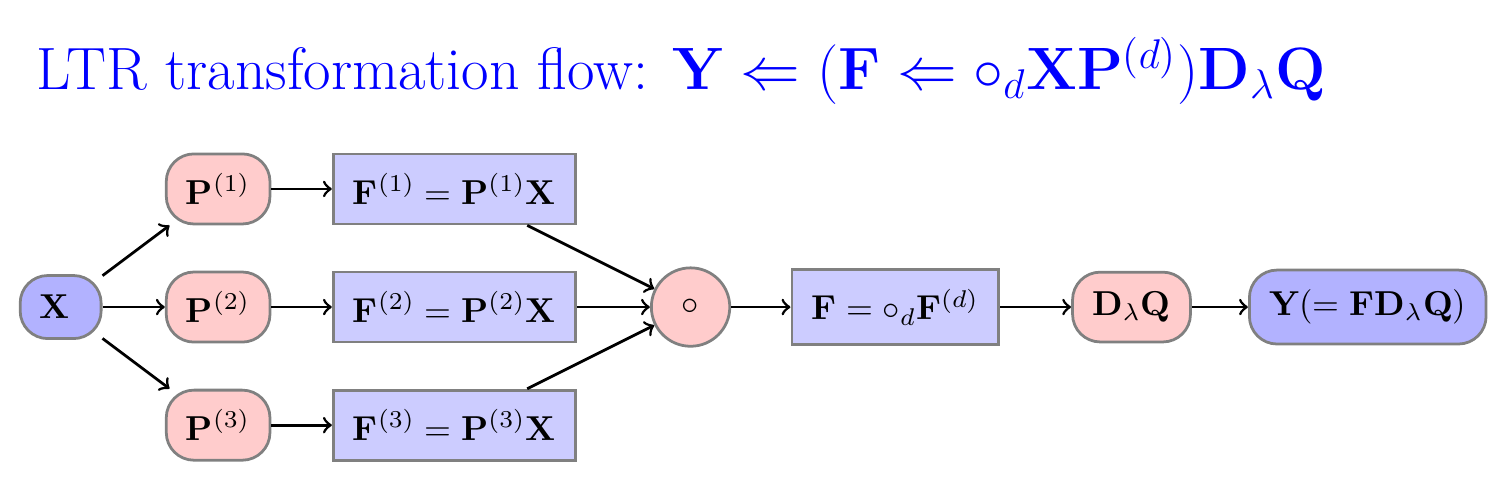}
\end{center}
\caption{The transformation flow of the vector valued, matrix
represented LTR learning problem for degree $3$. }
\label{fig:transformation_flow}
\end{figure}

In the matrix representation of the optimization problem
(\ref{eq:ltr_matrix}), the predictor function contains only a matrix
product and pointwise products of the corresponding matrices, thus the
original tensor decomposition form completely disappears from the
computation. The transformation flow sending the input matrix into the
output one is summarized in Figure \ref{fig:transformation_flow}.

From that transformation flow and from the matrix expressions of the
gradient, the time complexity of one iteration can be computed, and
we can state 
\begin{proposition} 
At fixed number of epochs the time complexity of the gradient based algorithm is $O(mnn_dn_t)$.
\end{proposition}
\begin{proof}
Note that the dominating part in (\ref{eq:matrix_gradient}) is to compute the matrix products. Computing $\mbf{F}$ requires $O(mnn_dn_y)$ operations,  $\mbf{F}^{T} \mbf{E}$ and $\mbf{Q}\mbf{E}^{T}$ have the complexity $O(mn_tn_y)$, and finally the product  $(\mbf{F}_{\setminus d}^{T} \circ \mbf{D}_{\lambda}\mbf{Q}\mbf{E}^{T})\mbf{X}$ has complexity $O(mnn_t)$. Since these products can be evaluated sequentially, thus the complexity of one step is $O(mnn_dn_t )$. If the number of epochs is fixed, then this also gives the overall complexity of the entire algorithm. 
\end{proof}

\subsection{Multi-view case}

The basic case can be extended to deal with multiple input sources, where
the input matrix $\mbf{X}$ is replaced with a set of matrices
$(\mbf{X}_d| \mbf{X}_d \in \mathbb{R}^{m \times n_d})$ with potentially
different number of columns. Clearly, the optimization problem
(\ref{eq:ltr_matrix}) can be solved in the same way, thus this
extension is a very natural one. To implement the multi-view case, only
the loss term of the  objective function needs to be changed to $\left
\| \mbf{Y}- \tilde{\mbf{F}}\mbf{D}_{\lambda}\mbf{Q} \right \|_F^2$ where 
$\tilde{\mbf{F}} = \circ_{d=1}^{n_d}\mbf{X}_d \mbf{P}_d^{T}$.

\subsection{Classification problems}

In a classification problem, binary or multi-class, the functions connecting
the input and the output are generally discontinuous. 
To learn discontinuous functions the polynomial function $f(\mbf{x})=
\sum_{t=1}^{n_t}\lambda_{t}
\prod_{d=1}^{n_d}\braket{\mbf{p}^{(t)}_d,\mbf{x}}$ can be embedded
into a differentiable activation function $h: \mathbb{R} \rightarrow
\mathbb{R}$. By applying the chain rule it only adds a scalar factor
to the gradients computed for the polynomials. This type of activation
function can be applied component-wise in the vector valued prediction
as well.  

In a classification problem applying the logistic function, the LTR can
be embedded into  logistic regression by defining the conditional
probabilities as a smooth activation  
\begin{equation}
P(y=1|x)=\dfrac{1}{1+e^{-f(\mbf{x})}}
\end{equation}
where the original linear form is replaced with a multi-linear one.
Then the maximum likelihood problem corresponding to the extended
logistic regression can be straightforwardly solved by similar
gradient descent approach described for the least square regression case.

\section{Multi-layered model}

A multi-layered version of LTR can handle the trade-off between the complexity of
the polynomial function and the approximation error, more the rank
less the loss but the generalization performance might deteriorate.

The method presented here can be interpreted as a variant
of the general {\it gradient boosting} framework
\cite{Friedman99stochasticgradient}. Each problem related to a range
of ranks might be taken as a weak learner acting on the residue
remained after the aggregation of the approximation computed by the
previous learners.        

In this extension, the rank range $1,\dots,n_t$ are cut into blocks. 
Let $n_{1},\dots,n_{N_b}$ be a partition of the rank range
$[1,n_t]$ into disjoint index intervals, i.e. $n_t=\sum_{b=1}^{N_b}
n_{b}$. Let $n_{\le b} = \sum_{\tau=1}^{b} n_{\tau}$, and $n_{\le 0}=0$. With these notations the multi-layered algorithm can constructed.

\begin{enumerate}
\item Let $b=1$, and $\mbf{Y}_b=\mbf{Y}$. 
\item Initialize $\mbf{P}_d^{(b)}=\mbf{P}_d[n_{\le b-1}:n_{\le b},:]$ for
all $d=1,\dots,n_d$, 

$\mbf{Q}^{(b)}=\mbf{Q}[n_{\le b-1}:n_{\le b},:]$,
and $\sbf{\lambda}_b= \sbf{\lambda}[n_{\le b-1}:n_{\le b}]$.

\item Let $\mbf{F}_b= \circ_{d=1}^{n_d}\mbf{X}\mbf{P}_d^{(b)}$. 
\item Solve
\begin{equation*}
\begin{array}{ll}
\min \dfrac{1}{2m n_y}\left \| \mbf{Y}_b- \mbf{F}_b\mbf{D}_{\lambda_b}\mbf{Q}^{(b)}
\right \|_F^2 
+ \dfrac{C_{p}}{2n_{t_b} n_d n}\sum_{d=1}^{n_d}||\mbf{P}_d^{(b)}||_F^2 +
\dfrac{C_{q}}{2n_{t_b}n_y}||\mbf{Q}^{(b)}||_F^2\\   
\text{w.r.t.}\ 
\mbf{D}_{\lambda},\ \mbf{Q}^{(b)},\ \mbf{P}_1^{(b)},\dots,\mbf{P}_{n_d}^{(b)}.
\end{array}
\end{equation*}
\item Deflate output $\mbf{Y}_{b+1}=\mbf{Y}_b-
\mbf{F}_b\mbf{D}_{\lambda_b}\mbf{Q}^{(b)}$ 
\item Let $b=b+1$, If $b\le N_b$ Then Go To Step 2.
\item Provide function $\displaystyle f(\mbf{X})=\sum_{b=1}^{n_b}
\mbf{F}_b\mbf{D}_{\lambda_b}\mbf{Q}^{(b)}$. 
\end{enumerate}

The optimization problem in Step 4 can also be solved on a sequence of
mini-batches to keep the input source related memory complexity low.  



\begin{figure}[htbp]
\begin{center}
\begin{tabular}{cc}
  \includegraphics[width=3in,height=1.2in]{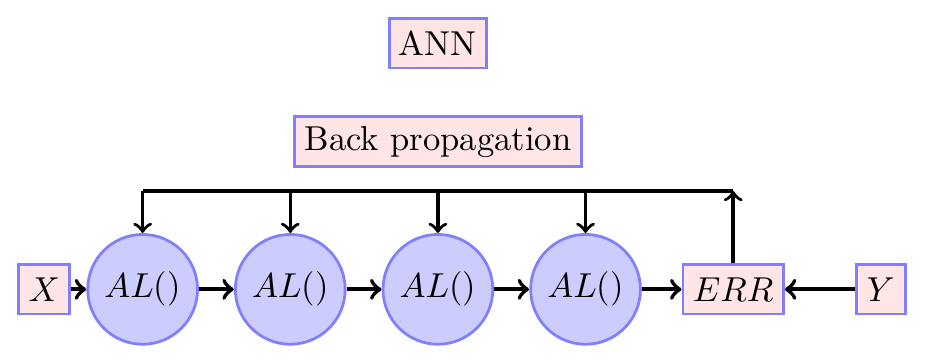}
& 
  \includegraphics[width=3in,height=1.2in]{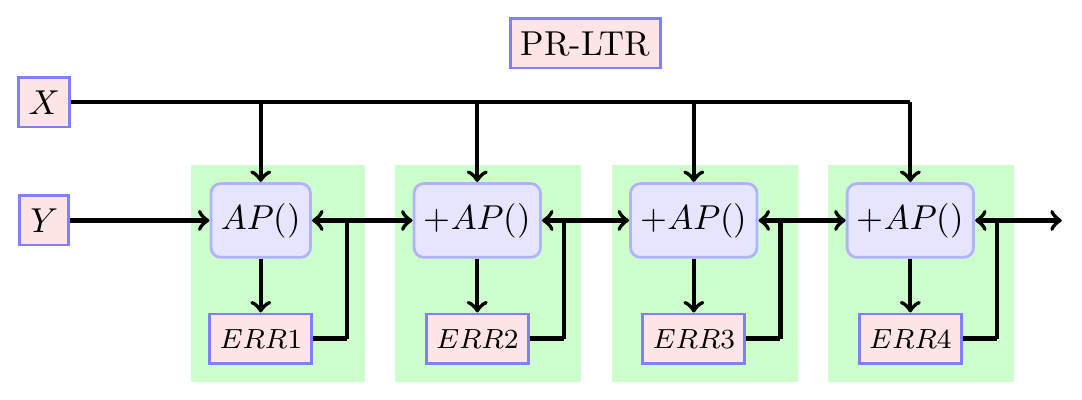}
\end{tabular}
\end{center}
\caption{A simple comparison of the layering strategy of a generic Neural Network and the LTR. The LTR reuses the input in every layer, the output is deflated, only the errors are propagated to the next layer, and error feedback is applied only within the layers. AL() denotes activation+linear layers in the ANN work flow, AP() stands for the activation+polynomial layers in the LTR case.}
\end{figure}




We could test that adding a new layer of ranks to those
accumulated earlier is a reliable step. By assuming that the output
is scalar valued then the reliability could be
measured by the {\it correlation ratio (intraclass correlation)},
\cite{Renyi2007}. 
Adding a new layer $t$ to increase the ranks might matter 
if the change in the output values highly correlate to
the sum of those computed by the previous ones. Let $y_{b,i}$ be
the predicted output at sample example $i$ in layer $b$. Then the
correlation ratio is given by
\begin{equation}
\displaystyle \eta^2= \dfrac{\sigma_b^2 n (\bar{y}_b-\bar{y})^{2}}{ \sum_{b=1}^{n_b}\sum_{i=1}^m
(y_{b,i}-\bar{y})^2},    
\end{equation}
where $\displaystyle \bar{y}_b=\frac{1}{m}\sum_{i=1}^{m}y_{b,i}$ is the layer wise
mean, and $\displaystyle \bar{y}= \frac{1}{n_b}\sum_{b=1}^{n_b}\bar{y}_b$ is the
total mean computed on all layers.



\section{Experiments}

\label{sec:experiments}

The experiments contain two main parts: in the first one, tests based
on random generated data sets are presented and in the second part,
results of experiments on real data sets are performed.      


The first collection of tests is generated on simple, two-variable
$x,y$, polynomials, see Table \ref{tbl:simplepoly_results} to
demonstrate the basic difference between LTR and the FM. Whilst LTR
uses unrestricted polynomials to fit the FM applies symmetric ones,
e.g. $\pi(x,y)=xy$, which condition limits the potential 
performance if the polynomial is arbitrary, e.g. anti-symmetric:
$x^2-y^2$.       

\begin{table}[htbp]
\begin{center}
\begin{tabular}{l|r|r|r|r}
& \multicolumn{4}{c}{Pearson-correlation (RMSE)} \\
Function & LR & KRR & FM & LTR \\ \hline
$xy$          & .01(1.00) & 1.0(.05) &1.00(.00)  & 1.0(.01) \\ 
$x^2-2xy+y^2$ & .04(2.80) & 1.0(.14) & .69(2.00) & 1.0(.02) \\ 
$x^2-y^2$     & .06(2.01) & 1.0(.09) & .05(2.01) & 1.0(.04) \\ 
\end{tabular}
\end{center}
\caption{ Learning simple quadratic functions where the input is generated independently from standard normal distribution {\it number of folds=5, m=1000, degree=2, rank=2, epoch=10}}
\label{tbl:simplepoly_results}
\end{table}

To test the general performance of the proposed learning method,
samples of randomly generated polynomials are used. The generation
procedure is based on  \ref{eq:function_dot_product}), where the
components of data vectors, $\mbf{x}_i$, the parameter vectors,
$\mbf{p}_d^{(t)}$, and rank-wised scalar factors, $\lambda_t$, are
chosen independently from standard normal distribution. This sample
could contain symmetric and non-symmetric polynomials as well. In the
experiments, 2-fold cross validation is used to keep the test size high
relative to that of the training, thus the distances between the test
examples to the closest training ones can be sufficiently large. The
results, the mean values, and standard errors are computed on all
random examples and at each random sample on all folds.  

In all tests, only one parameter runs through on a range of possible values
while all other ones are fixed. The data related parameters
which can influence the performance are: the degree and the rank of
the generated polynomials, the number of examples and the number of 
variables. The parameters of the learner are: the degree ($n_d$) and
the rank of the tensor ($n_t$) to be reconstructed, and the number of
epochs. Other learner parameters are fixed in all experiments,
learning speed is $1$, the size of the mini-batches is taken as $500$,
the regularization constants, $C_{p}$ and $C_{q}$ are fixed to
$10^{-5}$.  
An additional parameter, the noise level, is also used. The noise is
generated from Gaussian distribution with zero mean and the standard
deviation chosen as the standard deviation of the values of the random
polynomials multiplied with the noise parameter. All parameters
employed in the experiments are reported on the corresponding figures. 

The accuracy is measured by Pearson Correlation between the predicted
and the real values.  

\begin{figure}[t]
\begin{center}
\includegraphics[width=4in,height=2.5in]{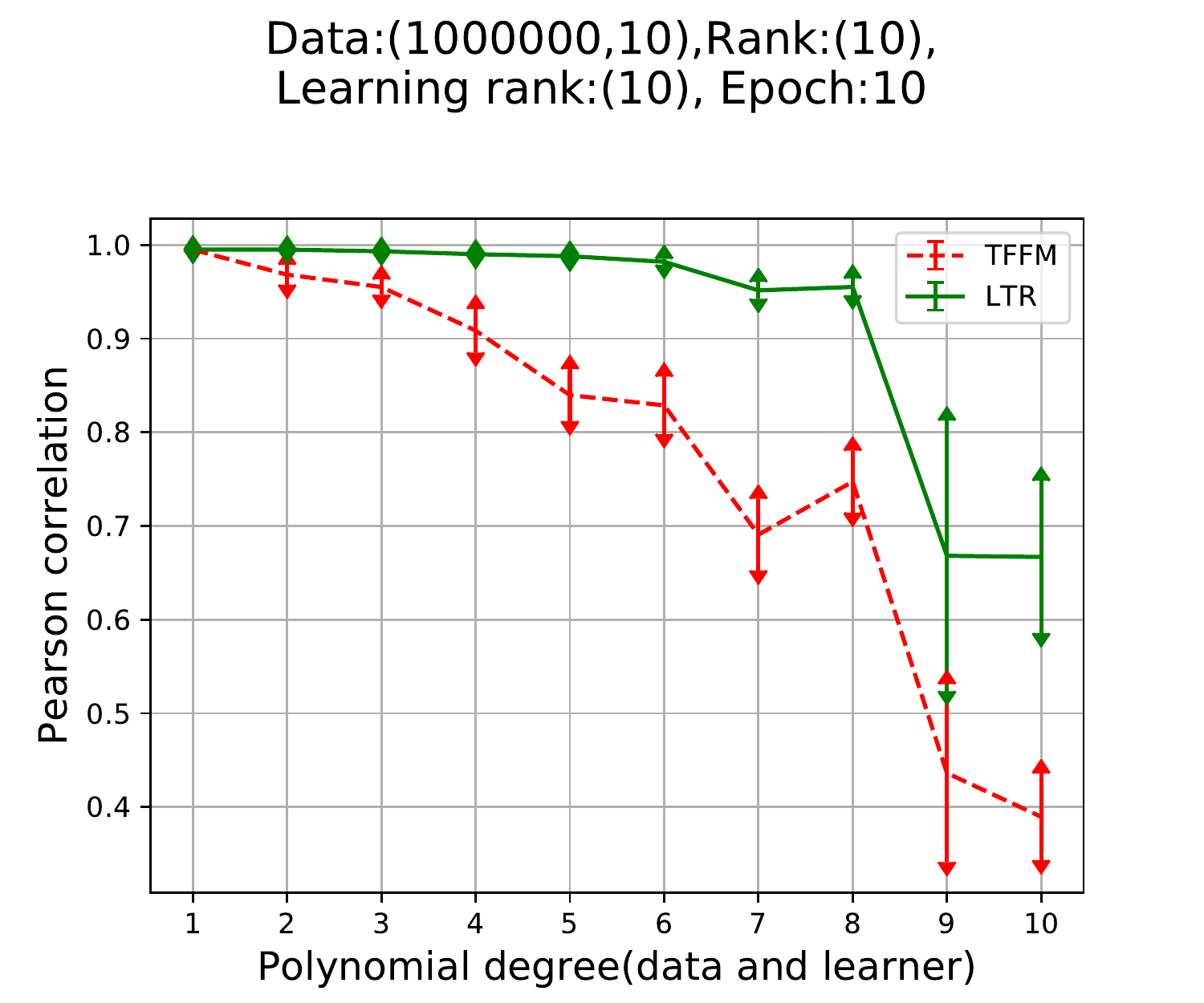}
\end{center}
\caption{Accuracy of the prediction as function of the degree of the data and learner ($n_d$) polynomials. Other parameters were fixed to $m=10^6, n=10$ and $n_t=10$.}
\label{fig_random_poly_1000000_degree_pcorr}
\end{figure}

\begin{figure}[htbp]
\begin{center}
\includegraphics[width=4in,height=2.5in]{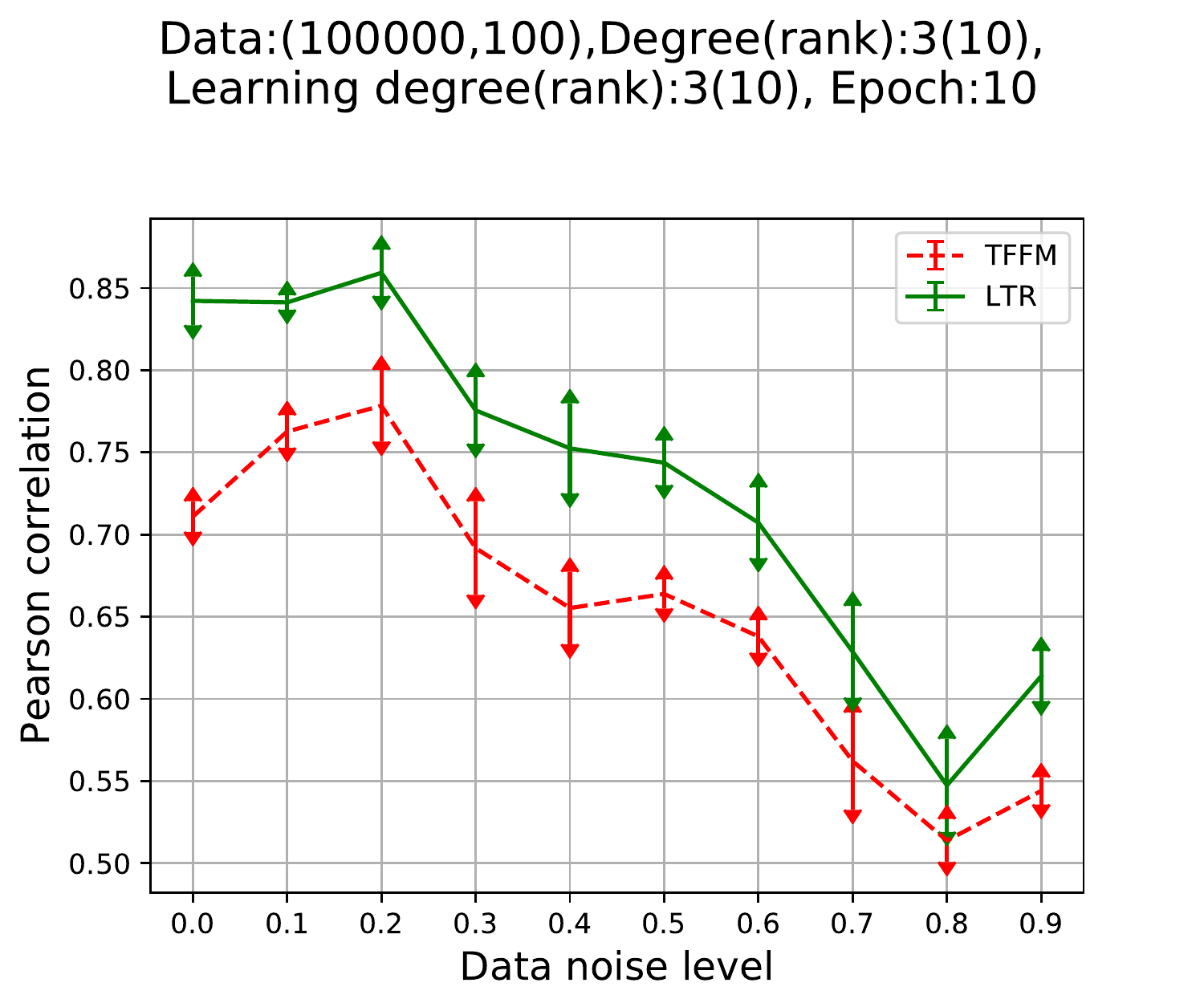}
\end{center}
\caption{Accuracy of the prediction as a function of the noise level. Other parameters were fixed to $m=10^5, n=10$, $n_d=3$ and $n_t=3$.}
\label{fig_random_poly_100000_noise_pcorr}
\end{figure}

\begin{figure}[htbp]
\begin{center}
\includegraphics[width=4in,height=2.5in]{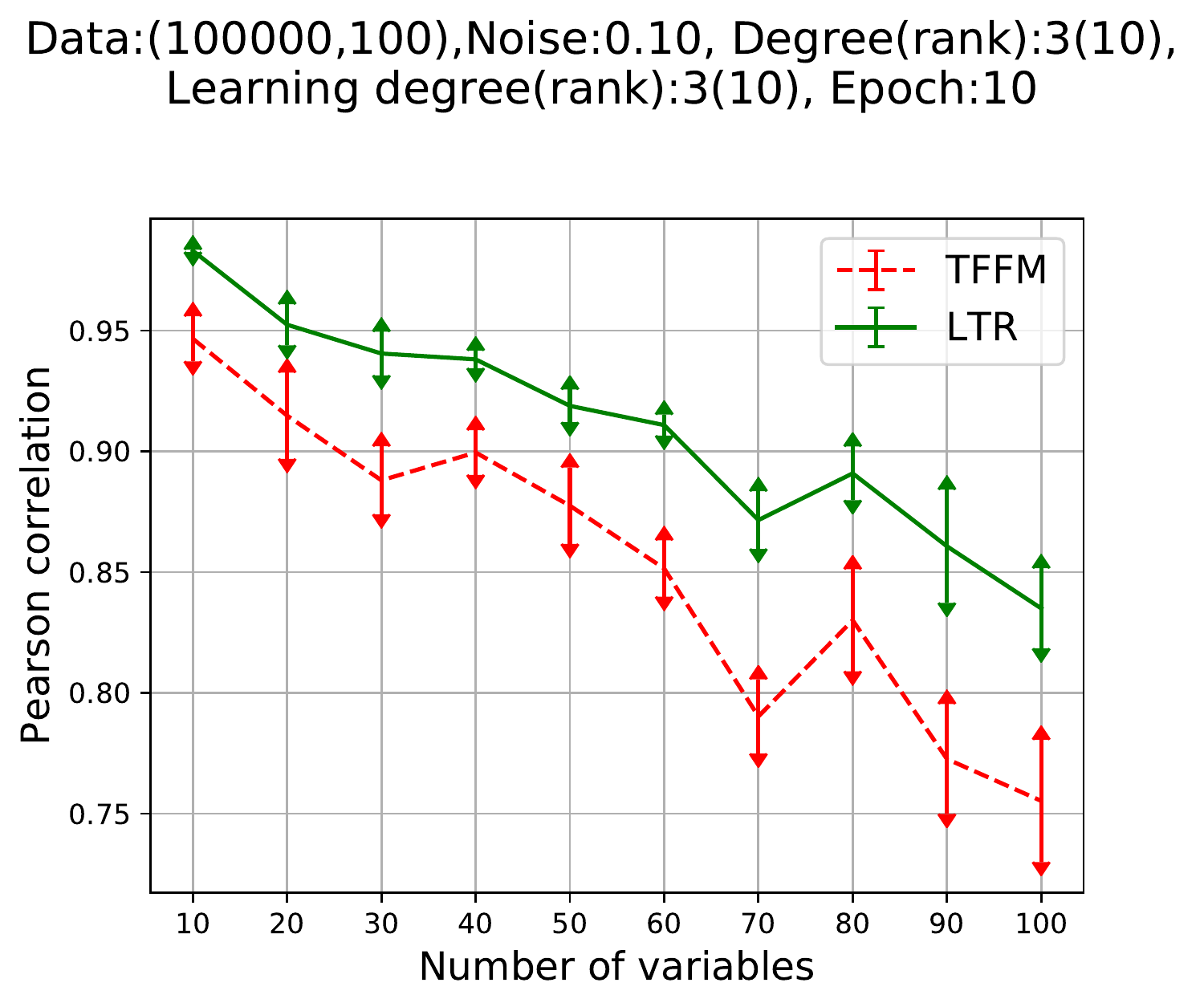}
\end{center}
\caption{Accuracy of the prediction as function of the number of variables. Other parameters were fixed to $m=10^5$, $n_d=3$ and $n_t=10$, noise level $=0.1$.}
\label{fig_random_poly_100000_variables_pcorr}
\end{figure}

\begin{figure}[t]
\begin{center}
\includegraphics[width=4in,height=2.5in]{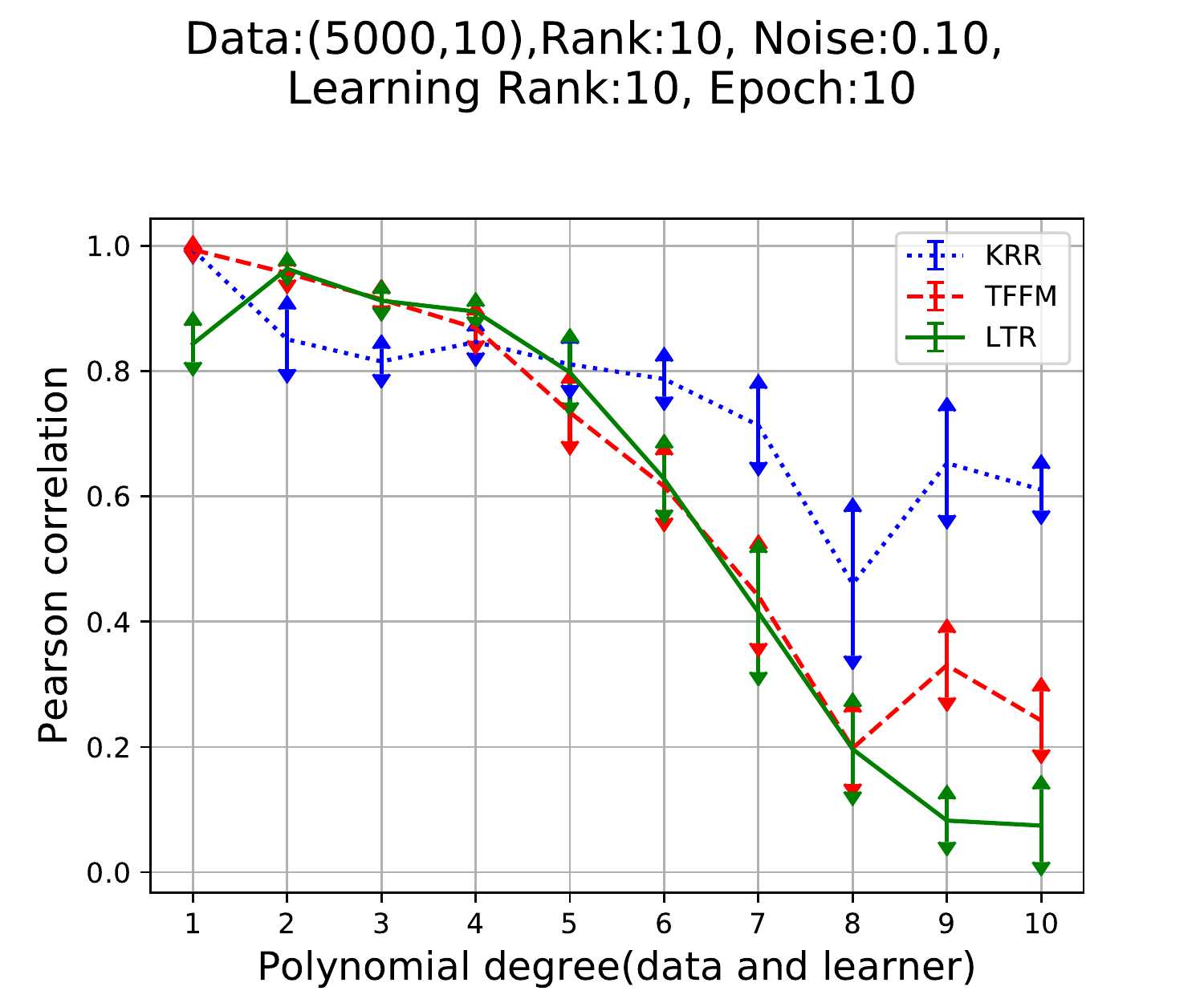}
\end{center}
\caption{Accuracy of the methods on a relatively small data set ($m=5000$) as a function polynomial degree. Other parameters were fixed to $n=10$, $n_d=10$ and $n_t=10$.}
\label{fig_random_poly_5000_degree}
\end{figure}

The methods compared are Kernel Ridge Regression with polynomial
kernels (KRR), High Order Factorization Machine implemented on
TensorFlow (TFFM) \cite{trofimov2016}, Linear Regression (LR), and the
proposed Latent Tensor Reconstruction method (LTR). In all
experiments, every method applies the same polynomial parameters, degree and
rank. In case of the TFFM the mini-batch size is chosen as the same as used by
the LTR. The methods are implemented in Python with the help of
Numpy package, except the FM where the model also uses the
TensorFlow. The tests were run on a HP-Z230 workstation equipped with 8
CPU's and 32GB memory. GPU has not been exploited in the experiments.    

\begin{figure}[htbp]
\begin{center}
\includegraphics[width=4in,height=2.5in]{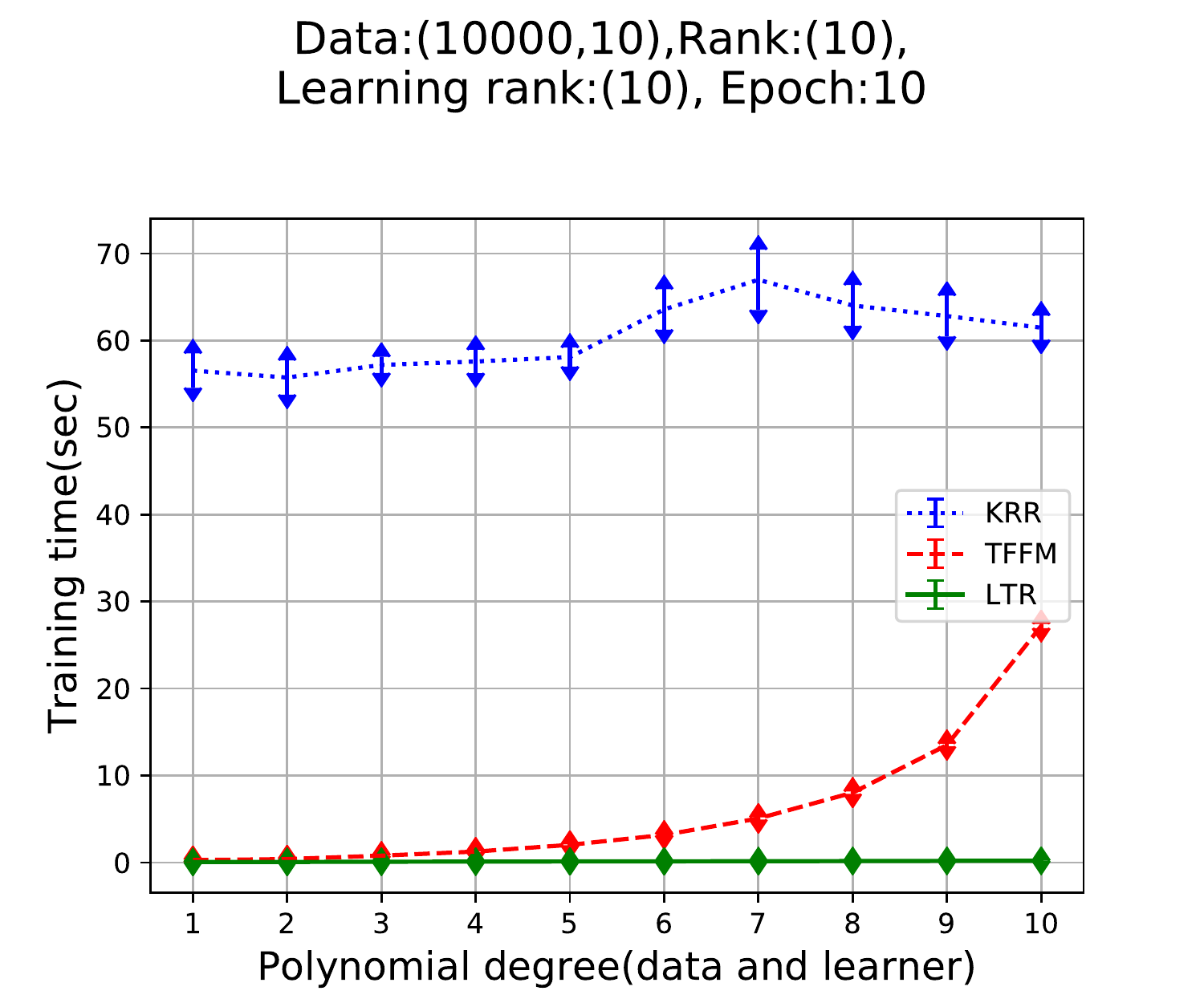}
\end{center}
\caption{The average execution time of the training to learn polynomials on random samples including Kernel Ridge Regression as a function of polynomial degree $(n_d)$. Other parameters were fixed to $m=10^4, n=10$, $n_d=10$ and $n_t=10$.}
\label{fig_random_poly_10000_degree_time}
\end{figure}

\begin{figure}[htbp]
\begin{center}
\includegraphics[width=4in,height=2.5in]{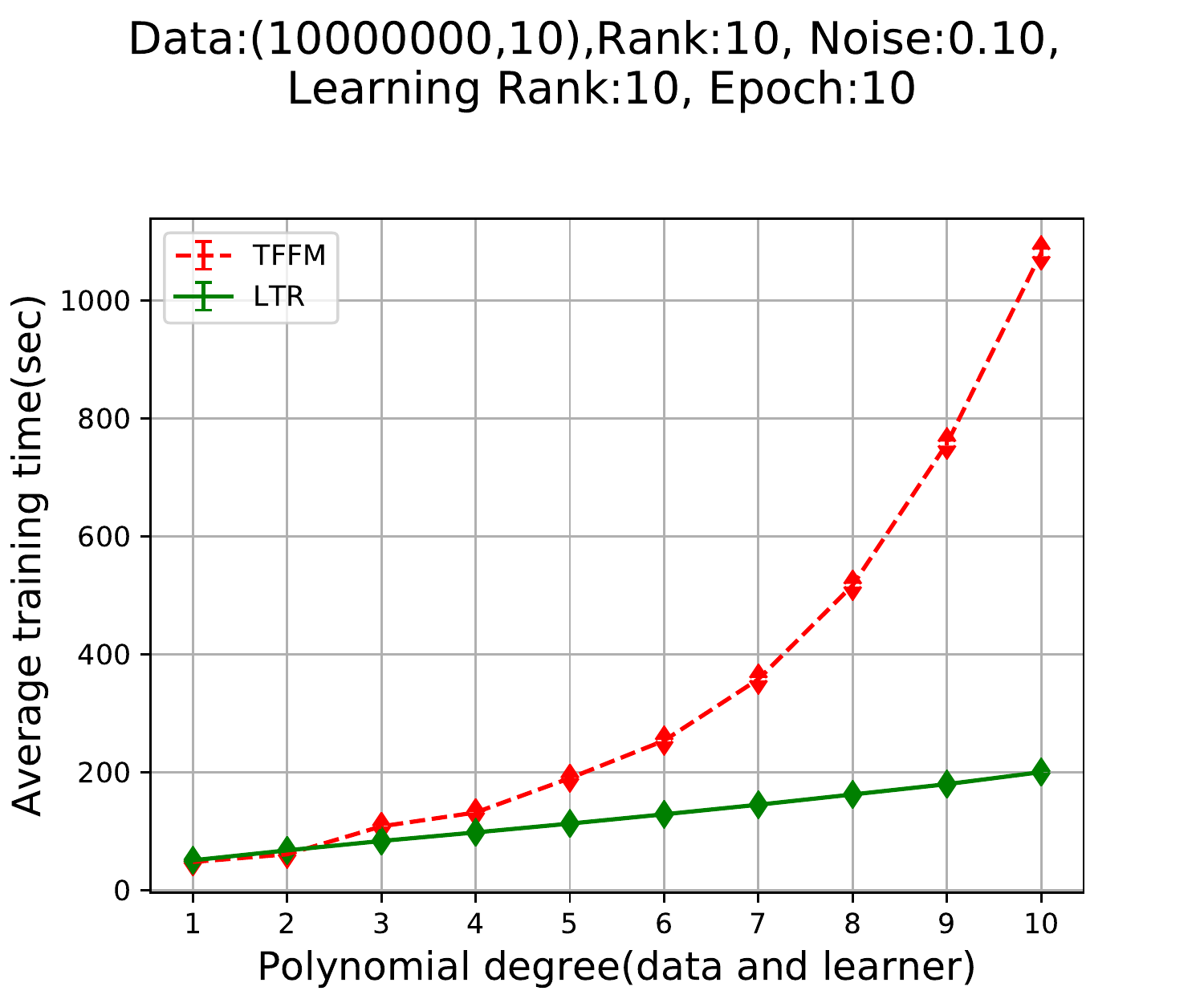}
\end{center}
\caption{The average execution time of the training to learn polynomials on large scale random samples ($m=10^7$) as a function of the polynomial degree ($n_d$). Other parameters were fixed to $n=10$ and $n_t=10$.}
\label{fig_random_poly_10000000_degree_time}
\end{figure}

The second part of the experiments is built on real data sets. 
We processed three data sets,  Corel5k, Espgame and Iaprtc12 \cite{guillaumin2010multiple}. They are available on the authors web site\footnote{http://lear.inrialpes.fr/people/guillaumin/data.php}.  
The published results on these data sets and the corresponding
references are included in Table 
\ref{tab:final_comparison}. These data sets consist of annotated images
and 15 preprocessed image features. The annotations are vectors of
labels identifying the objects appearing on the images. The image features,
e.g. SIFT, are provided by authors of
\cite{guillaumin2010multiple}. The high sparsity of the label vectors and the
strong interrelations between the feature variables make the prediction problem
rather challenging. This test is based on the vector output model
of the LTR presented in Section \ref{sec:vector_valued}.     

\begin{table}
\begin{center}
\begin{tabular}{l|r|rrr}
  &  & \multicolumn{2}{c}{Number of instances } & non-zero \\
Data set   & labels & training & test & labels $(\%)$\\ \hline
Corel5k   & 260 &  4500 &   500 & 1.31 \\ 
Espgame   & 268 & 18689 &  2081 & 1.75 \\
Iaprtc12  & 291 & 17665 &  1962 & 1.97 \\
\end{tabular}
\end{center}
\caption{The size of the data sets, and the number of labels assigned to each image}
\end{table}

\begin{table}[!t]
\begin{center}
\begin{tabular}{l|ccc}
   \multicolumn{4}{c}{Label accuracy in F1($\%$)} \\  
Method & Corel. & Espg. & Iapr. \\ \hline 
MBRM \cite{Feng_2004_CVPR}&  24.0 &  18.0 &  23.0 \\ 
TagProp \cite{guillaumin2009tagprop}        
&  37.0 &  {\bf 32.0} &  {\bf 39.0} \\
JEC \cite{survey}        &  29.0 &  21.0 &  23.0 \\
FastTag \cite{chen_2013_icml}&  37.0 &  30.0 &  34.0 \\ \hline
LTR                           &  {\bf 42.7} &  27.2 &  37.0 \\                
\end{tabular}
\end{center}
\caption{Comparison between of LTR and other
related methods on the three benchmark databases 
}
\label{tab:final_comparison}
\end{table}



\subsection{Summary of the experiments}

Figures \ref{fig_random_poly_1000000_degree_pcorr},
\ref{fig_random_poly_100000_noise_pcorr} and 
\ref{fig_random_poly_100000_variables_pcorr} present the better
accuracies provided by the LTR compared to the FM on varying degrees,
noise levels and on the number of variables, where the data
sets contain $10^5-10^6$ examples. These experiments show that the LTR
has a sufficient capacity to learn highly complex functions on large data
sets at a relatively low number of epochs ($10$). On small sets (Figure
\ref{fig_random_poly_5000_degree}) the large parameter space of both the
LTR and the TFFM leads to inferior performance compared to KRR, especially
when the number of variables is large.         

Figure \ref{fig_random_poly_10000_degree_time} shows that KRR has
a several times larger time complexity than the LTR or the TFFM. The
large scale example, (10M examples), in Figure
\ref{fig_random_poly_10000000_degree_time} demonstrates the
significant  
advantage of the LTR method which has a linear time algorithm in
the degree of the polynomials. That property is the most critical one in
capturing high level interdependence between the variables.


\section{Discussion}

In this paper, a latent tensor reconstruction based 
regression model is proposed for learning highly non-linear target 
functions from large scale data sets. We presented and efficient mini-batch
based algorithm that has a constant memory complexity and linear running time
in all meaningful parameters (degree, rank, sample size, number of
variables). These properties guarantee a high level scalability on
broad range of complex problems. Our experiments
demonstrate high accuracy in learning high-order polynomials from
noisy large-scale data, as well as competitive accuracy in some
real-world data sets.  



\section*{Acknowledgments}

%

This study was supported by the Academy of Finland [Grant No., 
311273, 313266,  
310107, 313268 and 334790].

\bibliography{ltr_arxiv.bib}

\begin{thebibliography}{10}

\bibitem{MR0385023}
Walter Rudin.
\newblock {\em Principles of mathematical analysis}.
\newblock McGraw-Hill Book Co., New York, third edition, 1976.
\newblock International Series in Pure and Applied Mathematics.

\bibitem{Kolda09tensordecompositions}
Tamara~G. Kolda and Brett~W. Bader.
\newblock Tensor decompositions and applications.
\newblock {\em SIAM REVIEW}, 51(3):455--500, 2009.

\bibitem{Lathauwer2000}
L.~De Lathauwer, B.~De Moor, and J.~Vandewalle.
\newblock A multilinear singular value decomposition.
\newblock {\em Journal of Matrix Anal. Appl.}, 21(4):1253--1278, 2000.

\bibitem{Golub2013}
G.~H. Golub and C.~F.~V. Loan.
\newblock {\em Matrix Computations}.
\newblock The Johns Hopkins University Press, Baltimore, MD, 4th edition
  edition, 2013.

\bibitem{Kaltofen:1990:CPG:77763.77768}
Erich Kaltofen and Barry~M. Trager.
\newblock Computing with polynomials given byblack boxes for their evaluations:
  Greatest common divisors, factorization, separation of numerators and
  denominators.
\newblock {\em J. Symb. Comput.}, 9(3):301--320, March 1990.

\bibitem{zbMATH00967945}
David {Cox}, John {Little}, and Donal {O'Shea}.
\newblock {\em {Ideals, varieties, and algorithms. An introduction to
  computational algebraic geometry and commutative algebra. 2nd ed.}}
\newblock New York, NY: Springer, 2nd ed. edition, 1996.

\bibitem{Rendle:2010:FM:1933307.1934620}
Steffen Rendle.
\newblock Factorization machines.
\newblock In {\em Proceedings of the 2010 IEEE International Conference on Data
  Mining}, ICDM '10, pages 995--1000. IEEE Computer Society, 2010.

\bibitem{Blondel:2016:HFM:3157382.3157473}
Mathieu Blondel, Akinori Fujino, Naonori Ueda, and Masakazu Ishihata.
\newblock Higher-order factorization machines.
\newblock In {\em Proceedings of the 30th International Conference on Neural
  Information Processing Systems}, NIPS'16, pages 3359--3367, USA, 2016. Curran
  Associates Inc.

\bibitem{Livni:2014:CET:2968826.2968922}
Roi Livni, Shai Shalev-Shwartz, and Ohad Shamir.
\newblock On the computational efficiency of training neural networks.
\newblock In {\em Proceedings of the 27th International Conference on Neural
  Information Processing Systems - Volume 1}, NIPS'14, pages 855--863,
  Cambridge, MA, USA, 2014. MIT Press.

\bibitem{Blondel:2016:PNF:3045390.3045481}
Mathieu Blondel, Masakazu Ishihata, Akinori Fujino, and Naonori Ueda.
\newblock Polynomial networks and factorization machines: New insights and
  efficient training algorithms.
\newblock In {\em Proceedings of the 33rd International Conference on
  International Conference on Machine Learning - Volume 48}, ICML'16, pages
  850--858. JMLR.org, 2016.

\bibitem{ESAllman2009}
Elizabeth~S. Allman, Catherine Matias, and John~A. Rhodes.
\newblock Identifiability of parameters in latent structure models with many
  observed variables.
\newblock {\em Project Euclid}, 37/6A, 2009.

\bibitem{JMLR:v15:anandkumar14b}
Animashree Anandkumar, Rong Ge, Daniel Hsu, Sham~M. Kakade, and Matus
  Telgarsky.
\newblock Tensor decompositions for learning latent variable models.
\newblock {\em Journal of Machine Learning Research}, 15:2773--2832, 2014.

\bibitem{JMLR:v16:huang15a}
Furong Huang, U.~N. Niranjan, Mohammad~Umar Hakeem, and Animashree Anandkumar.
\newblock Online tensor methods for learning latent variable models.
\newblock {\em Journal of Machine Learning Research}, 16:2797--2835, 2015.

\bibitem{Goodfellow:2016:DL:3086952}
Ian Goodfellow, Yoshua Bengio, and Aaron Courville.
\newblock {\em Deep Learning}.
\newblock The MIT Press, 2016.

\bibitem{Bellman1957}
R.~Bellman.
\newblock {\em Dynamic Programming}.
\newblock Princeton University Press, 1957.
\newblock Dover paperback edition (2003).

\bibitem{deSilva:2008:TRI:1461964.1461969}
Vin de~Silva and Lek-Heng Lim.
\newblock Tensor rank and the ill-posedness of the best low-rank approximation
  problem.
\newblock {\em SIAM J. Matrix Anal. Appl.}, 30(3):1084--1127, September 2008.

\bibitem{PMPrenter1970}
P.M. Prenter.
\newblock A {W}eierstrass theorem for real, separable {H}ilbert spaces.
\newblock {\em Journal of Approximation Theory}, 3:341--351, 1970.

\bibitem{Rendle:2012:FML:2168752.2168771}
Steffen Rendle.
\newblock Factorization machines with libfm.
\newblock {\em ACM Trans. Intell. Syst. Technol.}, 3(3):57:1--57:22, May 2012.

\bibitem{ShaweTaylor2004}
J.~Shawe-Taylor and Nello Cristianini.
\newblock {\em Kernel Methods for Pattern Analysis}.
\newblock Cambridge University Press, 2004.

\bibitem{NIPS2007_3182}
Ali Rahimi and Benjamin Recht.
\newblock Random features for large-scale kernel machines.
\newblock In J.~C. Platt, D.~Koller, Y.~Singer, and S.~T. Roweis, editors, {\em
  Advances in Neural Information Processing Systems 20}, pages 1177--1184.
  Curran Associates, Inc., 2008.

\bibitem{Friedman91multivariateadaptive}
Jerome~H. Friedman.
\newblock Multivariate adaptive regression splines.
\newblock {\em Ann. Statist}, 1991.

\bibitem{BTPolyak1964}
Boris~Teodorovich Polyak.
\newblock Some methods of speeding up the convergence of iteration methods.
\newblock {\em USSR Computational Mathematics and Mathematical Physics},
  4(5):1–17, 1964.

\bibitem{nesterov2005smooth}
Yu~Nesterov.
\newblock Smooth minimization of non-smooth functions.
\newblock {\em Mathematical programming}, 103(1):127--152, 2005.

\bibitem{Kingma2014}
Diederik Kingma and Jimmy Ba.
\newblock Adam: A method for stochastic optimization.
\newblock {\em International Conference on Learning Representations}, 12 2014.

\bibitem{Friedman99stochasticgradient}
Jerome~H. Friedman.
\newblock Stochastic gradient boosting.
\newblock {\em Computational Statistics and Data Analysis}, 38:367--378, 1999.

\bibitem{Renyi2007}
A.~R\'{e}nyi.
\newblock {\em Probability Theory}.
\newblock Dover, 2007.

\bibitem{trofimov2016}
Alexander~Novikov Mikhail~Trofimov.
\newblock tffm: Tensorflow implementation of an arbitrary order factorization
  machine.
\newblock \url{https://github.com/geffy/tffm}, 2016.

\bibitem{guillaumin2010multiple}
Matthieu Guillaumin, Jakob Verbeek, and Cordelia Schmid.
\newblock Multiple instance metric learning from automatically labeled bags of
  faces.
\newblock In {\em European conference on computer vision}, pages 634--647.
  Springer, 2010.

\bibitem{Feng_2004_CVPR}
S.~L. Feng, R.~Manmatha, and V.~Lavrenko.
\newblock Multiple bernoulli relevance models for image and video annotation.
\newblock In {\em CVPR}, 2004.

\bibitem{guillaumin2009tagprop}
Matthieu Guillaumin, Thomas Mensink, Jakob Verbeek, and Cordelia Schmid.
\newblock Tagprop: Discriminative metric learning in nearest neighbor models
  for image auto-annotation.
\newblock In {\em 2009 IEEE 12th international conference on computer vision},
  pages 309--316. IEEE, 2009.

\bibitem{survey}
Ameesh Makadia, Vladimir Pavlovic, and Sanjiv Kumar.
\newblock Baselines for image annotation.
\newblock {\em International Journal of Computer Vision}, 90:88--105, 2010.

\bibitem{chen_2013_icml}
Minmin Chen, Alice Zheng, and Kilian~Q. Weinberger.
\newblock Fast image tagging.
\newblock In {\em ICML}, 2013.

\end{thebibliography}
\bibliographystyle{unsrt}




\end{document}